\documentclass[letterpaper, 10pt, conference]{ieeeconf}  

\IEEEoverridecommandlockouts                              

\usepackage[english]{babel}

\makeatletter
\adddialect\l@ENGLISH\l@english
\makeatother

\usepackage{cite}

\usepackage{epsfig}
\usepackage{pst-all, graphics, graphicx, color}
\usepackage[crop=pdfcrop]{pstool}
\usepackage{psfrag}
\usepackage{subfigure}
\usepackage{amsmath, amssymb} 
\usepackage{epstopdf}
\usepackage{tikz}
\usepackage{multirow}
\usepackage{float}
\usepackage{algorithm, algorithmic}
\usepackage{dsfont}

\usepackage[makeroom]{cancel}

\floatstyle{ruled}
\newfloat{algorithm}{tbp}{loa}
\providecommand{\algorithmname}{Algorithm}
\floatname{algorithm}{\protect\algorithmname}

\makeatletter
\newcommand\footnoteref[1]{\protected@xdef\@thefnmark{\ref{#1}}\@footnotemark}
\makeatother

\renewcommand{\algorithmiccomment}[1]{\bgroup\hfill\scriptsize//~#1\egroup}
\newcommand{\algorithmicoutput}{\textbf{Output:}}
\newcommand{\OUTPUT}{\item[\algorithmicoutput]}

\def\tr{^{\rm T}}

\def\zero{\hbox{\bf 0}}

\def\bfa{{\mbox{\boldmath $a$}}}
\def\bfb{{\mbox{\boldmath $b$}}}
\def\bfc{{\mbox{\boldmath $c$}}}

\def\bfe{{\mbox{\boldmath $e$}}}
\def\bff{{\mbox{\boldmath $f$}}}
\def\bfg{{\mbox{\boldmath $g$}}}
\def\bfh{{\mbox{\boldmath $h$}}}

\def\bfn{{\mbox{\boldmath $n$}}}

\def\bfu{{\mbox{\boldmath $u$}}}

\def\bfx{{\mbox{\boldmath $x$}}}
\def\bfy{{\mbox{\boldmath $y$}}}

\def\bfA{{\mbox{\boldmath $A$}}}

\def\bfG{{\mbox{\boldmath $G$}}}

\def\bfI{{\mbox{\boldmath $I$}}}

\def\bfK{{\mbox{\boldmath $K$}}}

\def\bfN{{\mbox{\boldmath $N$}}}

\def\bfX{{\mbox{\boldmath $X$}}}

\newtheorem{theorem}{\bf Theorem}

\newtheorem{definition}{\bf Definition}
\newtheorem{remark}{\bf Remark}
\newtheorem{lemma}{\bf Lemma}

\begin{document}

\title{\LARGE \bf
Learning Barrier Functions for Constrained Motion Planning with Dynamical Systems}

\author{Matteo Saveriano$^{1}$ and Dongheui Lee$^{2,3}$%
\thanks{$^{1}$Intelligent and Interactive Systems and Digital Science Center (DiSC), University of Innsbruck, Innsbruck, Austria {\tt matteo.saveriano@uibk.ac.at}.}%
\thanks{$^{2}$Human-Centered Assistive Robotics, Technical University of Munich, Munich, Germany {\tt dhlee@tum.de}.}%
\thanks{$^{3}$Institute of Robotics and Mechatronics, German Aerospace Center (DLR), We{\ss}ling, Germany {\tt dongheui.lee@dlr.de}.}%
\thanks{This work was carried out while M. Saveriano was at the Institute of Robotics and Mechatronics of the German Aerospace Center (DLR). This work has been supported by Helmholtz Association.}
}


\maketitle


\begin{abstract}
Stable dynamical systems are a flexible tool to plan robotic motions in real-time. In the robotic literature, dynamical system motions are typically planned without considering possible limitations in the robot's workspace. This work presents a novel approach to learn workspace constraints from human demonstrations and to generate motion trajectories for the robot that lie in the constrained workspace. Training data are incrementally clustered into different linear subspaces and used to fit a low dimensional representation of each subspace. By considering the learned constraint subspaces as zeroing barrier functions, we are able to design a control input that keeps the system trajectory within the learned bounds. This control input is effectively combined with the original system dynamics preserving eventual asymptotic properties of the unconstrained system. Simulations and experiments on a real robot show the effectiveness of the proposed approach.  
\end{abstract}


\IEEEpeerreviewmaketitle

\section{Introduction}\label{sec:intro}
Modern robots are asked to perform complex tasks in unstructured environments and need to be endowed with a large set of pre-defined behaviors. Ideally, such behaviors are compactly represented, {generate motion trajectories in real-time, and generalize to different execution contexts. Moreover,} the list of behaviors should be easy to update. A trend in robotics suggests to use stable dynamical systems (DS) for planning of robotic skills \cite{DMP, SEDS, NeuralLearn2,  Clf, tau-SEDS, Perrin16, Blocher17}. DS can be eventually combined with Programming by Demonstrations (PbD) techniques \cite{Billard_PbD} to intuitively acquire novel skills.

Among the others, an interesting property of stable DS is their capability to plan converging motions from any point of the state space. At run time, the initial robot position is fed into the DS and the converging reference trajectory is automatically generated. However, {in constrained environments, planning with DS may fail because there is no guarantee that the generated motion satisfies the constraints}. 

This work focuses on motion planning with DS in a constrained workspace and, more specifically, in formally guaranteeing that generated robot trajectories do not violate the constraints (see Fig. \ref{fig:overview}). The idea is to consider the constraints on the robot motion as constraints on the state of the DS and to generate feasible trajectories that the low-level robot controller is able to track. We exploit a parameterized representation of the constraints and develop an approach to incrementally fit novel constraints from training data. In our framework, constraint representations are interpreted as zeroing barrier functions \cite{Ames17, Wieland07, Xu15, Rauscher16}. 
{Under the assumption that the constrained workspace is a convex set,} we are able to design a control input {that} prevents the DS trajectories to violate the constraints. The control input is computed by solving an optimization problem trying to minimize the influence of the resulting controller on asymptotic properties of the unconstrained DS. The optimization problem admits an analytical solution {that} makes the approach computationally efficient and effectively usable in robotic applications.

\begin{figure}[t!]
	\centering
	\includegraphics[width=0.9\columnwidth]{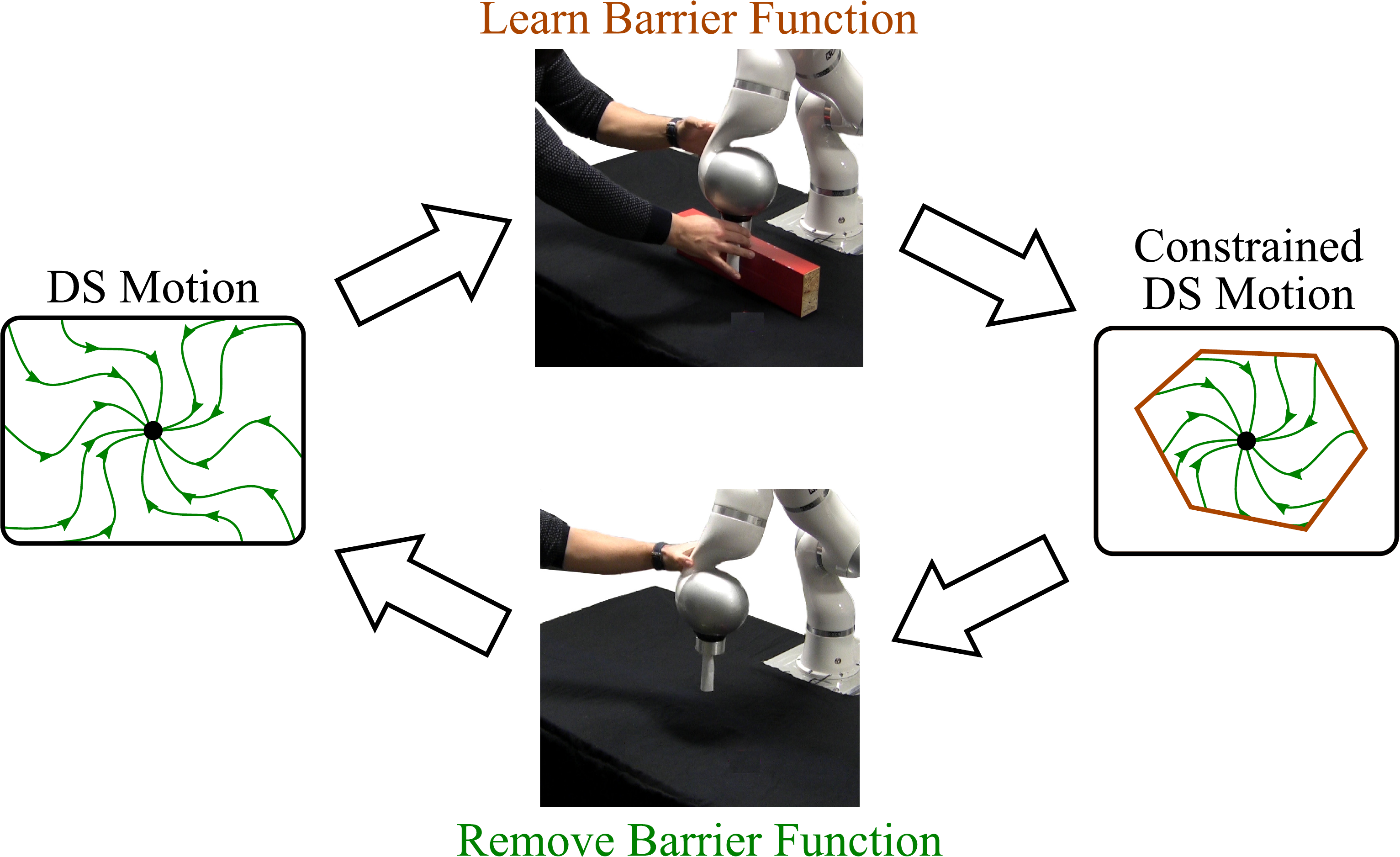}
	\caption{Overview of the proposed approach. Kinesthetic teaching is used both to add new constraints and to remove previously learned ones.}
	\label{fig:overview}
\end{figure}

The rest of the paper is organized as follows. Section \ref{sec:rel_work} summarizes related work in the field. Section \ref{sec:barrier} provides preliminary knowledge on DS and barrier functions. The extension to multiple barrier functions is discussed in Sec. \ref{sec:planning}. An approach for incremental learning of barrier functions is described in Sec. \ref{sec:learning}. {Comparison with existing approaches and} experiments on a real robot are presented in Sec. \ref{sec:experiments}. Section \ref{sec:conclusion} {concludes the paper}.

\section{{Related Work}}\label{sec:rel_work}
\subsection{Motion planning with dynamical systems}\label{sec:batch_ds_learn}
One of the first examples of motion planning with DS, dynamic movement primitives (DMPs) \cite{DMP} superimpose a non-linear forcing term to a linear dynamics to reproduce a demonstrated motion. The forcing term is suppressed by a time dependent signal to ensure convergence to the goal. DMPs can plan in joint or Cartesian space \cite{Ude14, Saveriano19} 
and can be parameterized to customize the execution \cite{TP-DMP, Pervez18}.

A limitation of DMPs is that stability is retrieved by canceling the forcing term after a certain time. The stable estimator of dynamical systems \cite{SEDS} {overcomes this limitation by using} a non-linear and autonomous DS, learned from demonstrations, and {constraining the DS parameters to ensure stability}. {Constraints adopted in \cite{SEDS}}  prevent to accurately learn non-linear motions. This accuracy problem is faced in several works \cite{NeuralLearn2, Clf, tau-SEDS, Perrin16}. For instance, \cite{tau-SEDS, Perrin16} {transform the dynamics in a diffeomorphed} space where demonstrations are accurately reproduced. {The work in \cite{NeuralLearn2} derives a set of stability constraints from a learned Lyapunov function, while \cite{Clf} uses the Lyapunov function to stabilize the DS}. A common result among \cite{NeuralLearn2, Clf, tau-SEDS, Perrin16} is that DS are able to plan {stable motions with complex shapes}.

Stable DS have several interesting properties {that} make them an effective tool for reactive motion planning. Complex dynamics are learned from demonstrations which greatly simplifies the DS design. DS generate converging motions from any point in the state space and the convergence is still preserved if the DS velocity is scaled with a positive scaling factor. 
DS trajectories can be modified to avoid possible collisions \cite{DS_avoidance, Saveriano13, Saveriano14, Saveriano17, Hoffmann09} or to realize an incremental learning of motor skills \cite{Nemec_15, Kulvicius13, Maeda17, Kronander15, Saveriano18,khoramshahi18}. DS parameters can encode kinematic and impedance behaviors \cite{Calinon10, SaverianoURAI14} and can be iteratively refined, e.g. {via} reinforcement learning \cite{Kormushev10, Buchli11}. 

Our {work} is complementary to {these} approaches. {We use} kinesthetic teaching to demonstrate the boundary of {the admissible set}. Given the constraints, the dynamics of a given DS---stable or unstable, manually designed or learned from demonstrations---are constrained within the admissible set while preserving eventual asymptotic properties of the DS.



\subsection{Constrained dynamics and motion planning}


DS are typically used in a robot and workspace agnostic manner. However, {robot motion} is subject to constraints deriving from physical limitations (e.g. velocity limits) or restrictions in the workspace (e.g. obstacles). {Our idea} is to consider constraints on the robot motion as constraints on the state of the DS. In this way, the constrained DS generates a feasible trajectory that the low-level robot controller is able to accurately track. The states of the DS {that} comply with the constraints belong to the {\textit{admissible set}. Since} DS trajectories may exit the admissible set, it is interesting to {design a control input that renders the set forward invariant}. In the control community, there are two main approaches for {this}, namely the \textit{invariance control} and the \textit{barrier functions}.

The invariance control \cite{Wolff04, Wolff05, Scheint08, Kimmel14} {renders the admissible set forward invariant with a discontinuous control}. The work in \cite{Zanchettin16} considers the robustness of the invariance control to uncertainties in the system dynamics or sensor measurements {using quadratic programming. The problem of invariance control is that} the computed control is discontinuous and may cause chattering. {As shown in \cite{Kimmel14}, chattering can be reduced}, but it is not guaranteed that chattering is avoided.   

Barrier functions \cite{Ames17, Wieland07, Xu15, Rauscher16} are capable of rendering the admissible set invariant with a smooth control action and therefore represent a suitable alternative to invariance control. In the literature, there are mainly two kinds of barrier functions. Reciprocal barrier functions are unbounded on the boundary of the admissible set, while zeroing barrier functions (ZBFs) vanish at the boundary. In this paper, we exploit ZBFs {because: \textit{i)}} having an unbounded function is undesirable while designing a controller for real-time or embedded systems \cite{Ames17} and {\textit{ii)}} the controller derived from the ZBF formulation is robust to system and sensor uncertainties \cite{Xu15}. {An experimental comparison between reciprocal and zeroing barriers is presented in Sec. \ref{subsec:comparison}.} 

{The problem of planning suitable robot motions can be alternatively seen as a path planning problem \cite{Choset2005, Bowen17}. Global path planners are capable of discovering feasible paths even in complex scenarios with high-dimensional configuration spaces and non convex constraint sets, while guaranteeing completeness and asymptotic optimality of the plan \cite{Bowen17}. On the contrary, the approach presented in this work applies to convex regions and Cartesian trajectories\footnote{{The approach naturally extends to joint space only if joint constraints form a convex set---e.g. box constraints on the joint angles.}}.  The drawback of path planners, compared with DS-based planners, is the higher computation time needed to plan the robot trajectory.
}

\section{{Zeroing Barrier Functions}}\label{sec:barrier}
\subsection{Preliminaries}
Consider the first-order DS {$\dot{\bfx} = \bff(\bfx)$, $\bfy = \bfh(\bfx)$}, where $\bfx, \dot{\bfx} \in \mathbb{R}^n$ are the state and its time derivative respectively (position and velocity of the robot), $\bff : \mathbb{R}^{n} \rightarrow \mathbb{R}^{n}$ is a locally Lipschitz 
non-linear function, and the output $\bfy$ is given by the vector of functions $\bfh(\bfx):\mathbb{R}^n\rightarrow\mathbb{R}^m$. A solution of $\dot{\bfx} = \bff(\bfx)$, namely $\bfx(\bfx_0,t) \in \mathbb{R}^{n}$, is called a trajectory. Different initial conditions $\bfx_0$ generate different trajectories. Let us denote with $M(\bfx_0)$ the maximal interval  of  existence of $\bfx(\bfx_0,t)$. The DS is said to be \textit{forward  complete} if $M(\bfx_0) = \mathbb{R}_0^+$ for any $\bfx_0 \in \mathbb{R}^{n}$. A set $\mathcal{D}\subseteq\mathbb{R}^{n}$ is \textit{forward  invariant} if for every $\bfx_0 \in \mathcal{D}$ the trajectory $\bfx(\bfx_0,t)  \in \mathcal{D}$ for all $t \in M(\bfx_0)$. A point $\hat{\bfx} : \bff(\hat{\bfx}) = \mathbf{0} \in \mathbb{R}^{n}$ is an equilibrium point. An equilibrium $\hat{\bfx} \in \mathcal{S} \subset \mathbb{R}^{n}$ is globally asymptotically stable if $\lim_{t\rightarrow+\infty} \bfx(\bfx_0,t) =\hat{\bfx}, \forall \bfx_0 \in \mathbb{R}^{n}$. 

{We focus on generating constrained and converging paths for the robot}. A constraint is defined by a smooth scalar function $h(\bfx)$. Using $h(\bfx)$, the set of admissible states $\mathcal{C}$ is defined as  $ \mathcal{C} = \{\bfx : h(\bfx) \geq 0 \}$; $\partial\mathcal{C} = \{\bfx : h(\bfx) = 0 \}$ is the boundary and $I(\mathcal{C}) = \{\bfx : h(\bfx) > 0 \}$ the interior of $\mathcal{C}$. 

\subsection{Zeroing barrier functions}
\begin{definition}
A continuous function $\alpha : (-b,a)\rightarrow (-\infty,\infty)$ is an extended class $\mathcal{K}$ function for some $a,b > 0$ if $\alpha(0)=0$ and $\alpha$ is  strictly increasing.
\end{definition}
\begin{definition}
A smooth function $h:\mathbb{R}^n\rightarrow\mathbb{R}$ is  a {\textbf{zeroing  barrier function}}  (ZBF) for $\dot{\bfx} = \bff(\bfx)$ and for  the admissible set $\mathcal{C}$, 
 if  there  exist an extended class $\mathcal{K}$ function $\alpha$ and a set $\mathcal{D}$ with $\mathcal{C} \subseteq \mathcal{D} \subset \mathbb{R}^n$ such that, for all $\bfx \in \mathcal{D}$, $L_f h(\bfx) \geq -\alpha(h(\bfx))$,
 where the Lie  derivative $L_fh(\bfx)=\frac{\partial h}{\partial\bfx}\bff$.
 \end{definition}
Note that $h$ is defined on a set $\mathcal{D}$ larger than $\mathcal{C}$ to consider  the  effects  of  model  perturbations \cite{Xu15}. 
In special cases, the condition on $L_fh(\bfx)$ simplifies to 
$L_f h(\bfx) \geq -\gamma h(\bfx), \quad \gamma >0$,
which is used in this work. The following lemma from \cite[Lemma 2]{Ames17} defines the conditions for the existence of $\gamma$.
\begin{lemma}
Assume that the compact admissible set $\mathcal{C}$ is nonempty for a smooth function $h(\bfx)$ and consider the DS $\dot{\bfx} = \bff(\bfx)$. If $\dot{h}(\bfx)>0$, $\forall\bfx \in \partial\mathcal{C}$, then for each $k\geq1$, there exists a constant $\gamma>0:\dot{h}(\bfx)\geq -\gamma h^k(\bfx), \forall \bfx \in I(\mathcal{C})$.
\end{lemma}
If a ZBF exists, then $\mathcal{C}$ is forward invariant \cite[Proposit. 1]{Ames17}.


\subsection{Zeroing control barrier functions (ZCBFs)}\label{sebsec:zcbf}
{If the set $\mathcal{C}$ is not forward invariant, it} can be made forward invariant by designing a suitable controller $\bfu \in \mathcal{U}$. Let us consider the control affine system
\begin{equation}
\dot{\bfx} = \bff(\bfx) + \bfG(\bfx)\bfu, \quad \bfy = \bfh(\bfx),
\label{eq:ds_syst_cont}
\end{equation}
where $\bfu \in \mathcal{U} \in \mathbb{R}^m$, $\bfG = [\bfg_1,\ldots,\bfg_m]$, $\bff : \mathbb{R}^{n} \rightarrow \mathbb{R}^{n}$  and $\bfg_i : \mathbb{R}^{n} \rightarrow \mathbb{R}^{n}$ are locally Lipschitz vector fields. For the controlled DS \eqref{eq:ds_syst_cont} it is possible to define a ZCBF.
\begin{definition}
A locally  Lipschitz function $h:\mathbb{R}^n\rightarrow\mathbb{R}$ is a zeroing control barrier function (ZCBF) for the DS \eqref{eq:ds_syst_cont} and the admissible  set $\mathcal{C} \subset \mathbb{R}^{n}$, 
 if the Lie derivatives $L_f h(\bfx)$ and $L_G h(\bfx)=[\frac{\partial h}{\partial \bfx}\bfg_1,\ldots,\frac{\partial h}{\partial \bfx}\bfg_m]$ are locally Lipschitz and there exists an extended class $\mathcal{K}$ function $\alpha$ such that
 $\sup_{\bfu \in \mathcal{U}} \left[L_f h(\bfx) + L_G h(\bfx)\bfu + \alpha(h(\bfx))\right] \geq 0, \forall\bfx \in \mathcal{C}.$
\end{definition}
By applying to \eqref{eq:ds_syst_cont} any Lipschitz continuous control $\bfu \in \mathcal{U}_h$, 
$\mathcal{U}_h = \left\lbrace \bfu \in \mathcal{U} : L_f h(\bfx) + L_G h(\bfx)\bfu + \alpha(h(\bfx)) \geq 0 \right\rbrace$, 
the set $\mathcal{C}$ becomes forward invariant \cite[Corollary 2]{Ames17}. The set $\mathcal{U}_h$ is non-empty if $L_G h(\bfx) \neq \zero$, meaning that $h(\bfx)$ has to be designed such that $\dot{h}(\bfx)$ has a direct dependence on $\bfu$. This is the case for constraints of relative degree one \cite{Isidori95}. 


\section{Motion Planning with Multiple ZCBF}\label{sec:planning}
In this section, we present an approach to enforce multiple constraints on a stable system used to plan robot motion. The DS is rendered globally asymptotically stable by a nominal control input $\bfu_o$. Hence, we derive an approach to merge the stabilizing controller with satisfaction of the constraints solving a quadratic program (QP). A similar problem is considered in \cite{Rauscher16}. The differences are that we use zeroing instead of reciprocal barrier functions and that we constraint the planned motion instead of the robot dynamics.    

\subsection{Multivariate ZCBF}\label{subsec:g_set}
Assume that $c$ constraints are specified {for \eqref{eq:ds_syst_cont} via} the smooth constraint functions $h_i(\bfx)\geq 0, i=1,\ldots,c$. The set of admissible {states {that} fulfill $h_i(\bfx)$} is defined as 
\begin{equation}
\begin{split}
     \mathcal{G} & = \{\bfx : h_i(\bfx) \geq 0, \forall i=1,\ldots,c \}\\
     \partial\mathcal{G} & = \left\lbrace\bfx : \exists 1\leq j \leq c: h_j(\bfx) = 0 \right.  \\ 
    &\qquad\quad\left. \wedge ~h_i(\bfx) \geq 0, \forall i=1,\ldots,c, \right\rbrace \\
     I(\mathcal{G}) & = \{\bfx : h_i(\bfx) > 0, \forall i=1,\ldots,c  \}, 
\end{split}
\label{eq:g_set}
\end{equation}
which is the intersection of the sets $\mathcal{C}_i$ associated with the constraints $h_i(\bfx)$. The set $\mathcal{G}$ is non-empty only if the individual constraints do not conflict. Assume that each $h_i(\bfx)$ is a ZCBF for the set $\mathcal{C}_i$. Any Lipschitz continuous control $\bfu \in \mathcal{U}_{h_i}$, with $\mathcal{U}_{h_i}$ defined as in Sec. \ref{sebsec:zcbf}, makes $\mathcal{C}_i$ forward invariant. The set of admissible control values for   $\mathcal{G}$ is given by the intersection of all $\mathcal{U}_{h_i}$
\begin{equation}
\label{eq:control_set_multi}
\mathcal{U}_{\small\bfh\normalsize} = \{ \bfu \in \mathcal{U}: L_f h_i(\bfx) + L_G h_i(\bfx)\bfu + \gamma_i h_i(\bfx) \geq 0 \},
\end{equation}
where $i=1,\ldots,c $, and $\gamma_i > 0$ are tunable gains associated to each $h_i(\bfx)$. If the set $\mathcal{G}$ is non-empty, i.e. the constraints do not conflict, $L_G h_i(\bfx) \neq \zero$, and all state variables of the system \eqref{eq:ds_syst_cont} are controllable, the following result holds.
\begin{theorem}
\label{th:inv_set}
Assume that the set $\mathcal{G}$ defined by \eqref{eq:g_set} is non-empty and that $h_i(\bfx)$ are the ZCBFs for the sets $\mathcal{C}_i$, for each $i=1,\ldots,c$. Then the set $\mathcal{G}$ can be rendered forward invariant by applying any Lipschitz continuous controller $\bfu \in \mathcal{U}_{\small\bfh\normalsize}$ to the system \eqref{eq:ds_syst_cont}.
\end{theorem}
\begin{proof}
The set $\mathcal{U}_{\small\bfh\normalsize}$ is non-empty and therefore a $\bfu \in \mathcal{U}_{\small\bfh\normalsize}$ exists. If $\bfu \in \mathcal{U}_{\small\bfh\normalsize}$, it holds that $\bfu \in \mathcal{U}_{h_i}, i=1,\ldots,c$. This is because the set $\mathcal{U}_{\small\bfh\normalsize}$ is obtained by intersecting all the sets $\mathcal{U}_{h_i}$. Hence, $\bfu$ renders each $\mathcal{C}_i$ forward invariant \cite[Corollary 2]{Ames17}. Being the intersection of $\mathcal{C}_i$ invariant sets an invariant set,  $\bfu \in \mathcal{U}_{\small\bfh\normalsize}$ renders $\mathcal{G}$ forward invariant.
\end{proof}

\subsection{Controller design}
Theorem \ref{th:inv_set} allows to use any control input $\bfu \in \mathcal{U}_{\small\bfh\normalsize}$ to render {$\mathcal{G}$} forward invariant. For states belonging to $I(\mathcal{G})$, where the constraints are not violated, the DS should be controlled by the stabilizing input $\bfu_o$ to achieve desired performance. A possible way to combine these control objectives consists in solving the {quadratic program} 
\begin{equation}
\label{eq:optimal_problem}
{\min_{\bfu \in \mathbb{R}^m} \Vert \bfu - \bfu_o \Vert^2, \quad \text{s.t.} \quad \bfA(\bfx)\bfu \preceq \bfb(\bfx)},
\end{equation}
where the symbol $\preceq$ indicates element-wise inequalities, and
   \begin{equation}
\label{eq:A_n}
\begin{split}
\bfA(\bfx) &= -[L_G h_i] = -\begin{bmatrix} L_{g_1} h_1 & \cdots & L_{g_m} h_1 \\
                                           \vdots & & \vdots \\
											L_{g_1} h_c & \cdots & L_{g_m} h_c
							\end{bmatrix} \\							
\bfb(\bfx) &= [L_f h_i + \gamma_i h_i] = \begin{bmatrix}L_f h_1 +  \gamma_1 h_1 \\
                                           \vdots \\
											L_f h_c +  \gamma_c h_c
							\end{bmatrix}.							
\end{split}
\end{equation}

\begin{theorem}\label{th:optimal_solution}
Consider  \eqref{eq:optimal_problem} and the control affine system \eqref{eq:ds_syst_cont} subject to $c$ state constraints with associated zeroing control barrier functions $h_i(\bfx), i=1,\ldots,c$. Assume that the rows of $\bfA(\bfx)$ associated to active inequality constraints are linearly independent and that a locally Lipschitz control input $\bfu_o \in \mathcal{U}$ is given for \eqref{eq:ds_syst_cont}. If the set $\mathcal{U}_{\small\bfh\normalsize}$ in \eqref{eq:control_set_multi} is non-empty, then the control law $\bfu^*(\bfx)$ obtained by solving the optimization problem \eqref{eq:optimal_problem} is Lipschitz continuous and renders the set $\mathcal{G}$ forward invariant.
\end{theorem}
\begin{proof}
Let us first prove that the solution $\bfu^*(\bfx)$ of \eqref{eq:optimal_problem} is unique and Lipschitz continuous. The rows of $\bfA(\bfx)$ associated to active inequality constraints are linearly independent by assumption. Moreover, it is assumed that the set $\mathcal{U}_{\small\bfh\normalsize}$ is non-empty, which implies that $L_G h_i(\bfx) \neq \zero$. Recalling that each row of $\bfA(\bfx)$ has the form $L_G h_i(\bfx) = \frac{\partial h_i}{\partial \bfx}\bfG(\bfx)$, we conclude that the gradients of the active constraints are linearly independent. Hence, the linear independence constraint qualification is satisfied and the problem \eqref{eq:optimal_problem} admits a unique solution \cite{Nocedal06}. The unique solution of \eqref{eq:optimal_problem} is Lipschitz continuous if $\bfA(\bfx)$, $\bfb(\bfx)$, and $\bfu_o$ are Lipschitz continuous at $\bfx$ \cite[Theorem 1]{Morris13}. $\bfu_o$ is locally Lipschitz by assumption. $\bfA(\bfx)$ and $\bfb(\bfx)$ are Lipschitz continuous for all $\bfx \in \mathcal{G}$ because $h_i$, $L_f h_i$, and $L_G h_i$ are locally Lipschitz (see Definition 3). Hence, the solution $\bfu^*(\bfx)$ of \eqref{eq:optimal_problem} is unique and Lipschitz continuous for all $\bfx \in \mathcal{G}$.

Given that $\bfu^*(\bfx)$ is Lipschitz continuous and that the constraints in \eqref{eq:optimal_problem} force $\bfu^*(\bfx)$ to lie in $\mathcal{U}_{\small\bfh\normalsize}$, the set $\mathcal{G}$ is forward invariant according to Theorem  \ref{th:inv_set}.
\end{proof}
\begin{remark}
Theorem \ref{th:optimal_solution} can be eventually proved using the weaker Mangasarian--Fromovitz constraint qualification \cite{Nocedal06}. The proof is analogous to \cite[Theorem 1]{Rauscher16}, but considering ZCBFs instead of reciprocal control barrier functions. 
\end{remark}
\begin{remark}
The control input computed by solving \eqref{eq:optimal_problem} renders $\mathcal{G}$ forward invariant even if the DS \eqref{eq:ds_syst_cont} is unstable. Asymptotic properties of the unconstrained DS are preserved if their objectives do not conflict with the forward invariance. For instance, the convergence to an equilibrium $\hat{\bfx}$ is preserved if $\hat{\bfx}$ is an admissible state, i.e. $\hat{\bfx} \in \mathcal{G}$. 
\end{remark}
\begin{remark}
The analytic solution of \eqref{eq:optimal_problem} is 
$\bfu^* = \bfA_c\left(\bfA_c \bfA_c\tr \right)^{-1}(\bfb_c - \bfA_c \bfu_o) + \bfu_o$,
where $\bfA_c$ contains only the rows of $\bfA$ and the vector $\bfb_c$ only the elements of $\bfb$ corresponding to active constrains.
\end{remark}

\subsection{Simulation results}\label{subsec:simulation}
We present some results on simulated DS to show the capabilities of the approach for constrained motion generation with multiple ZCBF and to clarify some important aspects of Theorem \ref{th:optimal_solution}. We consider $3$ planar dynamical systems, namely a stable DS obtained with the approach in \cite{SEDS} (Fig. \ref{fig:ds_simulation}(a)), the unstable DS $\dot{\bfx} = 2\bfx$ (Fig. \ref{fig:ds_simulation}(b)), and the stable limit cycle in Fig. \ref{fig:ds_simulation}(c)-(d) defined as $\dot{x}_1 = x_2 - x_1 (x_1^2 + x_2^2 -1)$, $\dot{x}_2 = -x_1 - x_2 (x_1^2 + x_2^2 -1)$.    
The nominal control input is $\bfu_o = [0,0]\tr$ and the matrix $\bfG = \bfI_{2\times2}$ for all the considered systems. The set of admissible states $\mathcal{G}$ is the intersection of $4$ linear constraint functions in the form $h_i = n_{i,1}x_1 + n_{i,2}x_2+a_i, i=1,\ldots,4$. Each DS has an equilibrium point (stable for the first, unstable for the others) at $\hat{\bfx}=[0,0]\tr$, {that} is an admissible state by construction. The $i$-th constraint is active if $L_fh_i(\bfx) + \gamma_i h_i(\bfx) \leq 0$. We set $\gamma_i = 10, i=1,\ldots,4$. The control input $\bfu^*$ that renders $\mathcal{G}$ forward invariant is computed as in Remark 3.

Results in Fig. \ref{fig:ds_simulation} show that, in all the considered cases, the  control input $\bfu^*$ renders the set $\mathcal{G}$ forward invariant. This is an expected result since the assumptions of Theorem \ref{th:optimal_solution} are matched. In Fig. \ref{fig:ds_simulation}(a), we notice that all the trajectories asymptotically converge to the stable equilibrium $\hat{\bfx}$. Indeed, being $\hat{\bfx}$ an admissible set, Remark 2 holds. Intuitively, we can say that if the unconstrained system is stable or stabilized through control at $\hat{\bfx}$, moving along the frontier of $\mathcal{G}$ sooner or later the DS generates a flow (velocity) pointing towards the interior of $\mathcal{G}$ where only the stabilizing input $\bfu_o$ is active. Then, the DS is driven towards the equilibrium point. Figure~\ref{fig:ds_simulation}(b) shows that $\mathcal{G}$ is forward invariant also for originally unstable dynamics (Remark 2). In this case, we do not expect that the control input $\bfu^*$ guarantees asymptotic convergence to a unique equilibrium, but we know that trajectories starting in $\mathcal{G}$ are bounded to $\mathcal{G}$. Finally, Theorem \ref{th:optimal_solution} does not guarantee the existence of periodic orbits within $\mathcal{G}$, even if the original DS has a limit cycle trajectory {(Fig.~\ref{fig:ds_simulation}(c)-(d))}. Depending on the shape of the constraints and the limit cycle dynamics, the constrained system may exhibit a periodic trajectory within $\mathcal{G}$ (Fig.~\ref{fig:ds_simulation}(c)) or not (Fig.~\ref{fig:ds_simulation}(d)).  
\begin{figure}[t!]
	\centering
	\includegraphics[width=\columnwidth]{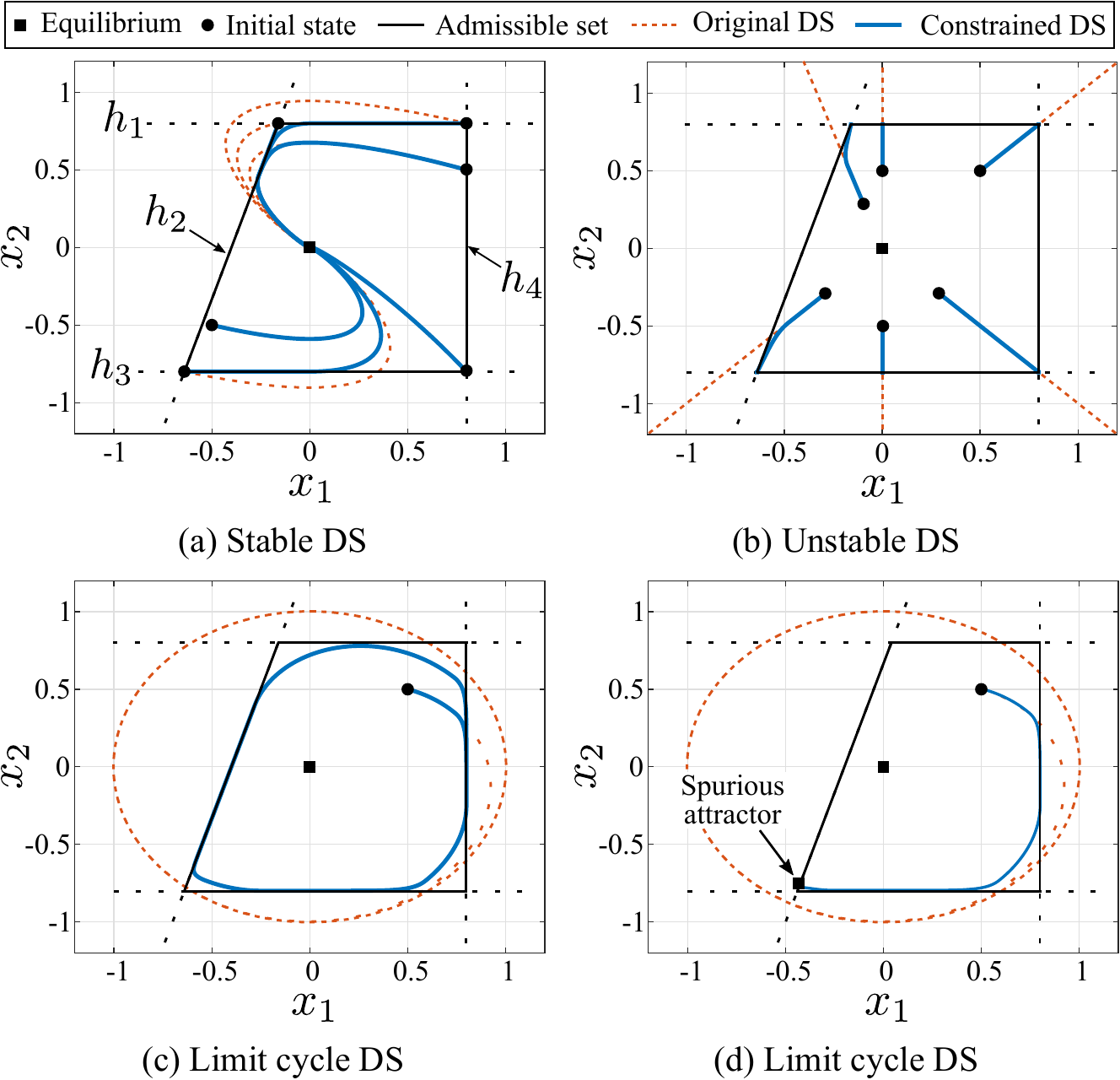}
	\caption{Results obtained by constraining the dynamics of three different systems within a convex region. {The four barrier functions are: $h_1 = x_1 - 0.8$, $h_2 = -x_1 + 0.3x_2 - 0.4$, $h_3 = -x_1 + 0.8$, and $h_4 = x_2 - 0.8$.}}
	\label{fig:ds_simulation}
\end{figure} 

\section{Incremental Learning of Barrier Functions}\label{sec:learning}

\subsection{Parametric ZCBF}\label{subsec:linear_zcfb}
As discussed in Sec. \ref{sec:planning}, zeroing control barrier functions have to match several requirements. In order to define a suitable parametric form for a ZCBF, we assume that the constraints are represented with a linear function
\begin{equation}
\label{eq:planar_zcbf}
h(\bfx) = \bfn\tr \bfx + a,
\end{equation} 
where the normal vector {$\bfn \in \mathbb{R}^n$} and the scalar $a$ are learnable parameters. {In practice, a linear function as that in \eqref{eq:planar_zcbf} represents a straight line if $n=2$ or a (hyper-)plane if $ n> 2$}.  Under mild assumptions, the linear constraint function \eqref{eq:planar_zcbf} represents a  ZCBF. Indeed, it is easy to verify that $h(\bfx)$ has the following properties:
\begin{enumerate}
\item $h(\bfx)$ is Lipschitz continuous.
\item The gradient of $h(\bfx)$, namely $\bfn\tr$, is constant and therefore Lipschitz continuous.
\item The Lie derivatives $L_f h(\bfx) = \bfn\tr \bff$ and $L_G h(\bfx) = \bfn\tr \bfG$ are Lipschitz continuous since $\bff$ and $\bfG$ are  Lipschitz continuous by assumption.
\item $L_G h(\bfx) \neq \zero$ if $\bfG$ has full column rank.
\item The gradients of two active constraints $h_1(\bfx)$ and $h_2(\bfx)$ are linearly dependent iff $\bfn_1\tr \bfn_2 = \bfn_2\tr \bfn_1 = \pm 1$, i.e. the normal vectors $\bfn_1$ and $\bfn_2$ are (anti-)parallel.
\end{enumerate}
A set of admissible states $\mathcal{G}$, defined as the intersection of $c$ linear constraints \eqref{eq:planar_zcbf} satisfying properties 1)--5), can be rendered forward invariant according to Theorem \ref{th:optimal_solution}. Moreover, properties 1) and 2) always hold, while properties 3) and 4) depend on the given DS. Hence, an algorithm that learns the constraints has to simply check that the normal of the active constraints are not aligned for property 5) to hold.

\subsection{Incremental learning of ZCBFs}\label{subsec:incremental_zcbf}
We aim at simultaneously clustering the data representing state constraints into multiple subspaces and finding a low-dimensional embedding (the linear ZBCF) for each cluster of points. This problem is known as \textit{subspace clustering} \cite{Vidal11}. {An} algorithm for clustering point clouds into planar subregions is proposed in \cite{Donnarumma12}. {This approach assumes that {hundreds of points are available} for each frame. The set of points is clustered into different planes based on the point to plane distances. In our setting, one point (robot position) is available for each time step. Therefore, we propose an approach simpler and faster than \cite{Donnarumma12} that is well-suited for online applications. The proposed  clustering algorithm is described as follows} and summarized in Algorithm \ref{alg:learning}.

\begin{algorithm}[t]\caption{Incremental learning of multiple ZCBFs}\algsetup{linenodelimiter=.}
	\label{alg:learning}
    \begin{algorithmic}[1]
    \REQUIRE {$\bfx_{t}$}, {$\bfX$}, $H = \{h_i(\bfx)\}_{i=1}^c$, $\delta$ \COMMENT{training point at time $t$, matrix of unlabeled points, set of learned ZCBF, threshold}
    \IF{$H$ is empty}
		\STATE {Add an empty column to {$\bfX$} and set it equal to {$\bfx_{t}$}} 
	\ELSE
	    \IF{{$-\delta\leq h_i({\bfx_{t}}) \leq \delta $} and $h_j({\bfx_{t}}) >  \delta, \forall i \neq j$}
			\STATE ${\bfx_{t}}$ lies on $h_{i}$ (eventually refine $\bfn_i$ and $a_i$)
		\ELSIF{$h_i({\bfx_{t}}) < {-\delta}$ for any $i=1,\ldots,c$}
			\STATE Remove $h_i$ from $H$ \COMMENT{User pushed the robot beyond the barrier}
		\ELSIF{$h_i({\bfx_{t}}) > \delta, \forall i=1,\ldots,c$}
			\STATE Add {$\bfx_{t}$} to {$\bfX$} \COMMENT{The new point belongs to a new cluster}
		\ENDIF
	\ENDIF
	\IF{{$\bfX$} contains $n$ points}
		\STATE Fit $h_{c+1}(\bfx)$, add $h_{c+1}(\bfx)$ to $H$, empty $\bfX$
	\ENDIF
	\OUTPUT {$\bfX$}, $H$
    \end{algorithmic}
\end{algorithm} 


{
In our setting, a new training point $\bfx_t$ arrives for each time step $t$. Assume that at current time $t$ no barrier has been learned yet. In this case, {a new column is added to the matrix $\bfX$ and set equal to the training point $\bfx_t$ (lines $1$--$3$ in Algorithm \ref{alg:learning}). The algorithm waits for the next point.} As soon as the matrix $\bfX$ contains $n$ points, where $n$ is the dimension of the state space, it is possible to fit a linear ZCBF in the form of $h_1(\bfx) = \bfn_1\tr\bfx + a_1$.} To this end, we perform a principal component analysis on $\bfX$ and set {the normal vector} $\bfn_1 = \bfe_n$ and {the scalar} $a_1 = -\bfn_1\tr\bfc$, where $\bfe_n$ is the $n$-th principal component {of $\bfX$} and $\bfc$ is the centroid of the points in {$\bfX$}. To ensure that the goal $\hat{\bfx}$ is an admissible point, we check the sign of $h_1(\hat{\bfx})$ and set $\bfn_1 = -\bfn_1$ and $a_1 = - a_1$ if $h_1(\hat{\bfx})<0$ (from {$L_f h(\bfx) \geq -\gamma h(\bfx)$ with} $\gamma > 0$ and $L_f h_1(\hat{\bfx})=0$).  
The barrier function $h_1(\bfx)$ defines a set of admissible states $\mathcal{C}_1$. 

{
Let us now assume that the barrier function $h_1(\bfx)$ has been learned. At time step $t$ a new training point $\bfx_{t}$ arrives. We consider three possible cases: \textit{i)} $\bfx_t$ lies on the plane described by $h_1$, \textit{ii)} $\bfx_t$ lies outside $\mathcal{C}_1$, and \textit{iii)} $\bfx_t$ lies inside $\mathcal{C}_1$ but not on $h_1$. The three cases can be discriminated considering that $h_1(\bfx_{t}) = \bfn_1\tr\bfx_t + a_1$, with $\Vert \bfn_1 \Vert = 1$, is the signed distance between $\bfx_{t}$ and the plane described by $h_1$. Looking at the value of $h_1(\bfx_{t})$, we have three cases:}
 \begin{enumerate}
 \item[\textit{i)}] $h_1(\bfx_{t}) = 0$ implies that {$\bfx_{t}$ lies on the plane described by} $h_1$. We can eventually refine $\bfn_1$ and $a_1$ by using incremental principal component analysis as in \cite{Donnarumma12}. 
 \item[\textit{ii)}] $h_1(\bfx_{t}) < 0 $ implies that {$\bfx_{t}$} is outside $\mathcal{C}_1$ and is a non-admissible state. Being $\mathcal{C}_1$ forward invariant, $h_1(\bfx_{t}) < 0 $ indicates that an external disturbance (the user) is pushing the robot beyond the barrier. We use this interaction modality to remove {$h_1$}. 
 \item[\textit{iii)}] $h_1(\bfx_{t}) > 0 $ implies that {$\bfx_{t}$} is an admissible point {($\bfx_{t} \in \mathcal{C}_1$)} that does not lie on the {plane described by} $h_1$. It is stored in the matrix {$\bfX$}.
\end{enumerate} 
{ 
In case \textit{iii)}, the point {$\bfx_t$} is admissible but distant from $h_1$. This is interpreted as the intention of the user to show a different barrier and $\bfx_t$ is stored in $\bfX$. If {$\bfX$} contains $n$ points, we fit a new barrier function $h_2(\bfx)$ with associated set of admissible states $\mathcal{C}_2$. Otherwise, the algorithm waits for the next point.} As discussed in Sec. \ref{subsec:g_set}, the new admissible set is the intersection of $\mathcal{C}_1$ and $\mathcal{C}_2$. The described procedure is repeated for every new training point. Note that case \textit{i)} avoids that two active constraints have their normals aligned. Hence, proposition 5) in Sec. \ref{subsec:linear_zcfb} is satisfied and the learned set of admissible states can be rendered forward invariant. 

{The algorithm generalizes to more than two barriers. Given $c$ barriers and a new point $\bfx_t$, we check if: \textit{i)} $\bfx_t$ is admissible and belongs to one of the $h_i$ (line $4$ in Algorithm \ref{alg:learning}), \textit{ii)} $\bfx_t$ is not admissible, i.e. $h_i(\bfx_t) < \delta$ at least for one $i$ (line $6$), or \textit{iii)} $\bfx_t$ is admissible and far from all the barriers $h_i(\bfx_t) < \delta$, $\forall i=1,\ldots,c$ (line $8$). Note that the threshold $\delta > 0$ is used in Algorithm \ref{alg:learning} to discriminate between the three cases. The threshold is needed in real cases to cope with possible uncertainties like noise in the measurements or imperfect robot control. The value of $\delta$ is determined considering that  $h_i(\bfx_t)$ represents the distance between $\bfx_t$ and the $i$-th barrier. The threshold also triggers the creation of a new barrier if points are distant from existing barriers. In this way, the boundary of a non-linear domain is approximated with a set of linear functions (Fig. \ref{fig:learn_alg_res}(a)). Finally, the algorithm works properly if $n$ consecutive training points represent the surface of the boundary. Otherwise, the barrier may deviate from the boundary as in Fig. \ref{fig:learn_alg_res}(b). In Sec. \ref{sec:rob_exp}, kinestetic teaching is used to collect proper training data.
\begin{figure}[t]
	\centering
	\subfigure[]{\includegraphics[width=0.35\columnwidth]{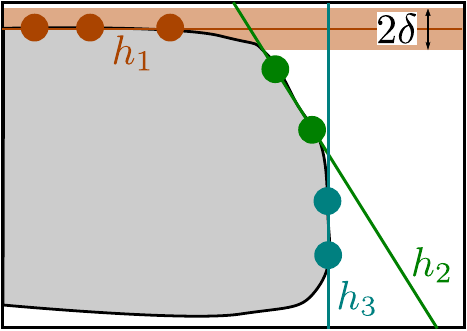}}
	\subfigure[]{\includegraphics[width=0.35\columnwidth]{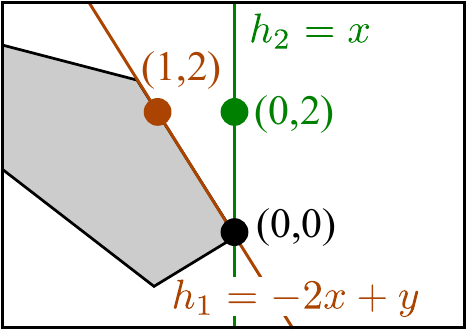}}
    	\caption{{(a) A non-linear boundary represented with several linear barriers. (b) Barriers learned with two different sets of points. }}
    	\label{fig:learn_alg_res}
\end{figure} 
}

\section{Experimental Results}\label{sec:experiments}
\subsection{{Comparative results}}\label{subsec:comparison}
{The approach presented in Sec. \ref{sec:planning} renders a convex set forward invariant. The same problem can be solved with other techniques like Reciprocal Control Barrier Functions (RCBF) \cite{Rauscher16} and Constrained Optimization (CO). Hence, we present an experimental comparison to underline similarities and differences among the three approaches. We consider the SEDS system used  in Fig. \ref{fig:ds_simulation}(a) and the three barrier functions $h_1 = x_1 - 0.8$, $h_2 = -x_1 + 0.3x_2 - 0.4$, and $h_3 = -x_1 + 0.8$.} 
\subsubsection*{{ZCBF and constrained optimization (CO)}}
{In the considered scenario, it holds that $\bfu_o = \zero$ (stable DS) and $\bfG(\bfx) = \bfI$. The $3$ barriers form the convex set $\mathcal{C} =\{\bfx| \bfN\bfx\preceq\bfa\}$ with $\bfN=-[\bfn_1,\bfn_2,\bfn_3]\tr$ and $\bfa=[a_1,a_2,a_3]\tr$. To render $\mathcal{C}$ forward invariant, we compute a control input by solving
\begin{equation*}
\min_{\bfu \in \mathbb{R}^m} 0.5\Vert \bfu\Vert^2, ~
\text{s.t.} ~ \bfN\int_0^t \dot{\bfx}dt = \bfN\int_0^t \left(\bff(\bfx) + \bfu\right)dt \preceq \bfa.
\end{equation*}
This CO problem is numerically solved in Matlab using the interior-point algorithm \cite{Nocedal06}. As shown in Fig. \ref{fig:zcbf_vs_rcbf_vs_co}(a), ZCBF and CO work comparably well for trajectories starting within the admissible set, guaranteeing forward invariance and convergence to the goal. However, in the considered 2D setup, CO takes about $0.01\,$s to plan the next position while ZCBF, for which an analytic solution exists, takes less than a millisecond ($0.2\,$ms).  Results in Fig. \ref{fig:zcbf_vs_rcbf_vs_co}(b) show that both approaches drive trajectories starting outside  $\mathcal{G}$ within the set. However, the controller obtained with CO is more aggressive and pushes the state within $\mathcal{G}$ in one step with consequent high velocity. Although it is possible to consider bounds on the control input ($\bfu_{min} \preceq \bfu \preceq \bfu_{max}$), this will make the problem harder and increase the computation time. On the contrary, the controller obtained with ZCBF smoothly drives the state within $\mathcal{G}$, which is preferable in real applications.
}

\begin{figure}[t!]
	\centering
	\includegraphics[width=\columnwidth]{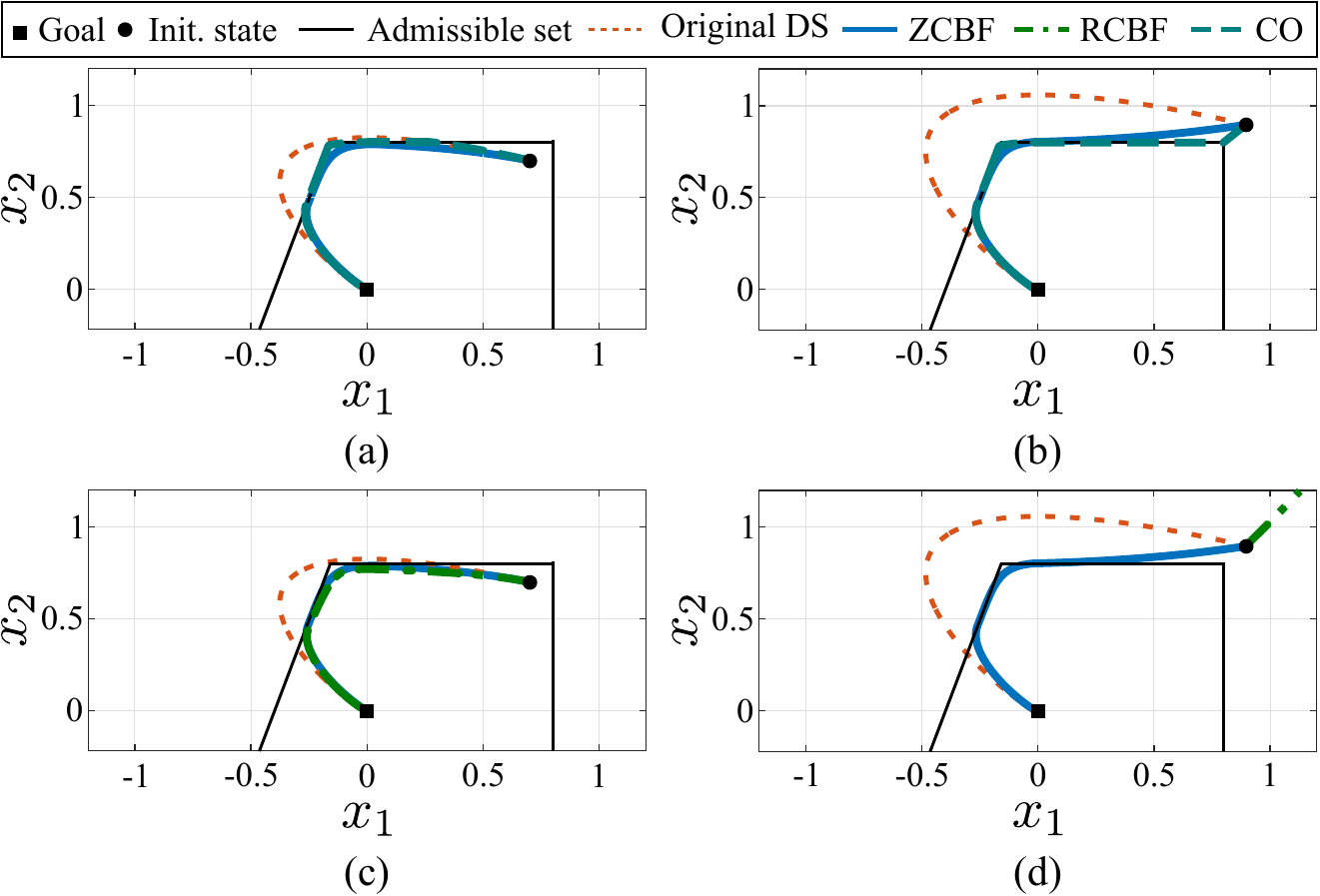}
	\caption{{Comparison of ZCBF with (a)--(b) Constrained Optimization (CO) and (c)--(d) Reciprocal Control Barrier Function (RCBF) approaches.}}
	\label{fig:zcbf_vs_rcbf_vs_co}
\end{figure} 

\subsubsection*{{ZCBF and RCBF}}
{
For RCBF, we use the common barrier $R_i = -\log\left( \frac{h_i(\bfx)}{1+h_i(\bfx)}\right)$ \cite{Rauscher16,Ames17} that diverges for $h_i \rightarrow 0$ and is undefined for $-1 \leq h_i \leq 0$. Results in Fig. \ref{fig:zcbf_vs_rcbf_vs_co}(c) show that the two approaches have a similar behavior for trajectories starting within the admissible set, guaranteeing forward invariance of  $\mathcal{G}$ and convergence to the goal. They also perform similarly in terms of computation time because also RCBF admits an analytic solution. The problem with RCBF appears when the trajectory starts outside the admissible set, where at least one of the $R_i$ is undefined. In Fig.  \ref{fig:zcbf_vs_rcbf_vs_co}(d), both $R_1$ and $R_3$ are undefined because $h_1, h_3<0$. This causes the trajectory to diverge from the admissible set. In real applications, uncertainties due to noise or imperfect control may drive the trajectory slightly outside the admissible set. In this cases, ZCBF smoothly drives the trajectory within  $\mathcal{G}$, showing its robustness to uncertainties. } 

\begin{figure*}[t!]
	\centering
	\includegraphics[width=0.96\textwidth]{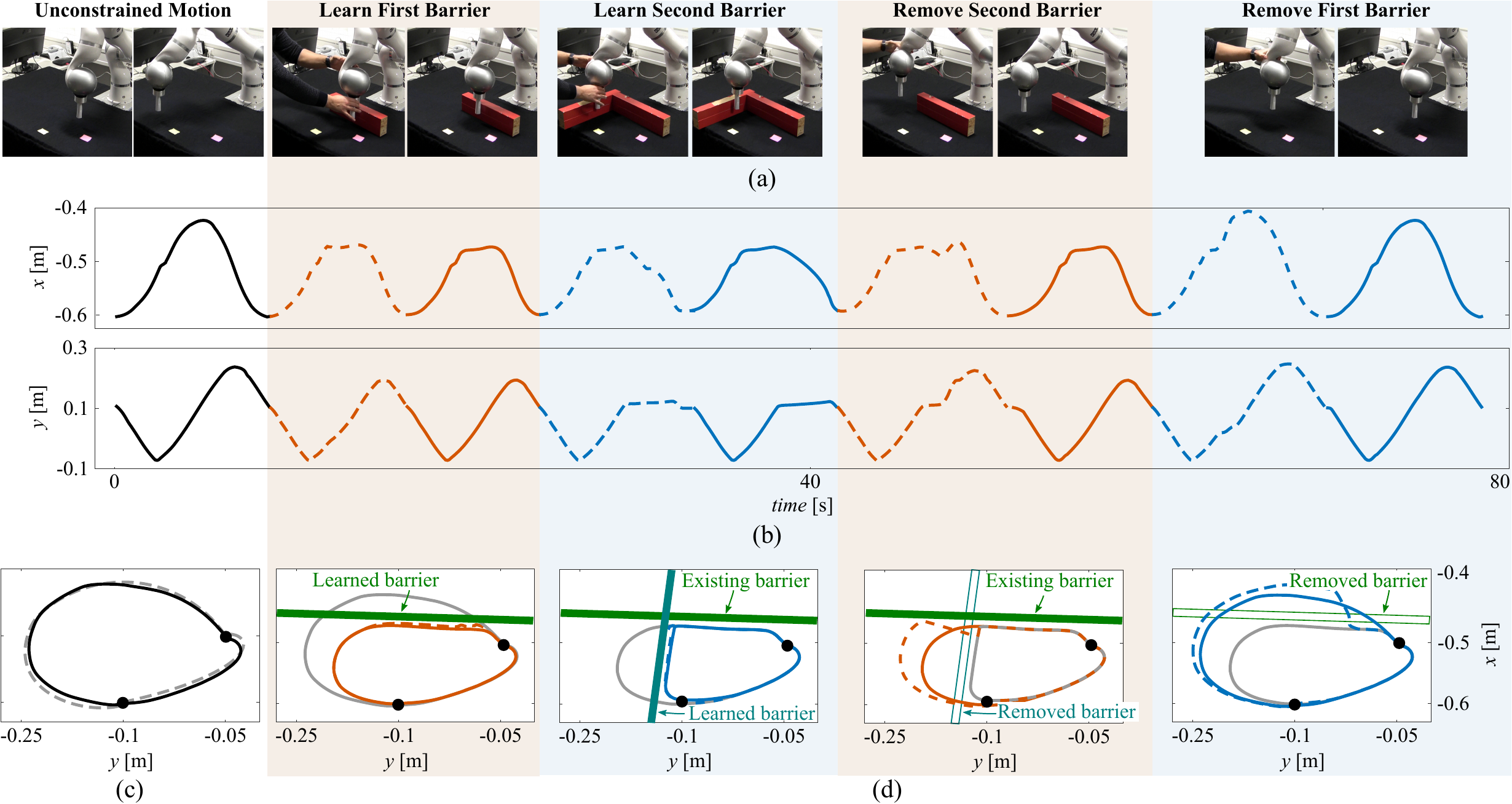}
	\caption{Incremental learning of zeroing barrier functions. (a) Snapshots of the incremental learning process. (b) The robot position ($x$ and $y$ over time) during the experiment. Dashed lines indicate kinesthetic teaching phases, solid lines represent the robot motion without external disturbances. (c) A comparison of the desired (grey dashed line) and executed (solid black line) trajectory in the unconstrained case. (d) The robot position in the $x$-$y$ plane during the experiment. Dashed lines indicate kinesthetic teaching phases, solid lines represent the robot motion after the teaching, and solid grey lines represent the robot motion before the teaching. Note that only $x$ and $y$ coordinates are shown because the learning process has no effect on $z$.}
	\label{fig:result}
\end{figure*} 

\subsection{Robot experiments}\label{sec:rob_exp}
\subsubsection*{{Setup and data collection}}
The proposed approach is tested on a KUKA LWR IV+ manipulator controlled at $500\,$Hz through the fast research interface (FRI) \cite{FRI}. The robot position is controlled with a Cartesian impedance controller, while the orientation is kept fixed. The stiffness matrix is $\bfK = k\bfI_{3\times 3}$, where $k=1000$ to have a relatively stiff robot that accurately tracks the desired trajectory (see Fig. \ref{fig:result}(c)). {External torques applied to the robot during the teaching are estimated by the FRI}. In case a non-zero torque is detected, the measured position in each time $\bfx_r$ is commanded to the robot. This allows the user to distract the robot from the DS trajectory and to start collecting training data. When the teaching ends (no external torque), the DS generates on-line a new trajectory starting from the current robot position. As shown in Fig. \ref{fig:result}(b), with this procedure the robot smoothly switches between teaching and execution phases. Training data are sampled at $10\,$Hz, i.e. we store one position every $50$ {samples}, to eliminate points {that} are too close to each other. Alternatively, one can collect only points at a certain distance. {Algorithm \ref{alg:learning} is used to add/remove barriers}. At run-time, the DS state is constrained by the {computed control} to remain within the admissible set.

\subsubsection*{{Results}}
The results of the proposed approach are shown in Fig. \ref{fig:result}. The desired robot trajectory in the $x$-$y$ plane is generated with the SEDS system used in Sec. \ref{subsec:simulation} (Fig. \ref{fig:ds_simulation}(a)), while the $z$ motion is generated with the linear DS $\dot{z} = \hat{z}-z$. The robot loops between the two goal positions $\hat{\bfx}_1 = [-0.5,0,-0.05]\tr$ and $\hat{\bfx}_2 = [-0.6,0,0.1]\tr$, depicted as black dots in Fig. \ref{fig:result}(c) and (d), following a periodic motion. This is obtained by switching the desired goal after reaching the previous one. The top row of  Fig. \ref{fig:result} shows snapshots of the interactive learning procedure. Kinesthetic teaching is effectively used to add new constraints or to remove existing ones, while the control based on the barrier functions renders the admissible set forward invariant generating a smooth constrained trajectory (middle and bottom rows of Fig. \ref{fig:result}). As stated in Remark 2, being the goals $\hat{\bfx}_1$ and $\hat{\bfx}_2$ admissible states, the constrained motion converges to the desired goal.

\subsection{Discussion}
Presented experiments show the effectiveness of our approach. The user demonstrates position constraints via kinesthetic teaching and linear barrier functions are incrementally learned. 
Moreover, the user can cancel a learned constraint by pushing the robot beyond the relative barrier. This is useful, for example, to cope with an erroneous demonstration, or to adapt the motion to a new scenario with different bounds. 


In this paper, we use the DS in the so-called open loop configuration where the DS is initialized with the robot state and the resulting motion is sent to the robot as a reference trajectory to track. Alternatively, one can use the DS in closed loop configuration by continuously feeding the measured state into the DS. This allows a continuous adaptation of the motion to external perturbations, but requires a customized controller to ensure stability or passivity of the closed loop system  \cite{Kronander16}. The presented approach does not directly apply to the closed loop configuration because the control input would not render the admissible set forward invariant. Indeed, the robot dynamics are not considered in $\bfu$ which may cause constraint violations. For the closed loop configuration, one has to directly modify the robot controller to ensure forward invariance of a given set. Although it is possible to constrain the robot controller \cite{Kimmel14, Zanchettin16}, the passivity of the closed loop system is not guaranteed. For this reason, we adopt the open loop configuration and leave {the extension to} closed loop configuration as a future work. 

The ZCBF formulation in Sec. \ref{sec:planning} is valid for linear and non-linear ZCFBs. In this work, we used planar ZCBFs for two reasons: \textit{i)} it is straightforward to satisfy the requirements of Theorem \ref{th:optimal_solution}, and \textit{ii)} training data are efficiently and incrementally clustered into planar regions. A limitation of the planar ZCBFs is shown in Fig. \ref{fig:problem}. To preserve the stability, the equilibrium point has to lie in the admissible set (Remark 2). This assumption can be violated if one wants to avoid a closed region like an obstacle. As shown in Fig. \ref{fig:problem}(b), the goal can be non-admissible and the motion stops at the boundary of the admissible set. Another possibility is that the starting point is non-admissible (Fig. \ref{fig:problem}(c)). In this case, the control input drives the state within the admissible set but the trajectory can enter the white area. To overcome this limitation, one can use a non-linear ZCBF and create a {concave} admissible set that contains starting and goal points but not the white region. {The possibility of planning within concave sets will also make possible to extend the approach to the joint space, where obstacles have typically non-convex shapes. Extending the approach to non-linear functions with a concave admissible set is left as future work.} 
\begin{figure}[t]
	\centering
	\subfigure[]{\includegraphics[width=0.32\columnwidth]{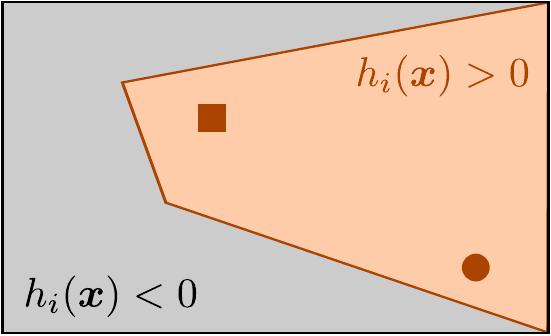}}
	\subfigure[]{\includegraphics[width=0.32\columnwidth]{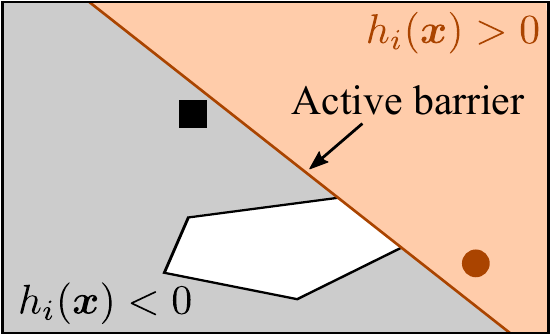}}
	\subfigure[]{\includegraphics[width=0.32\columnwidth]{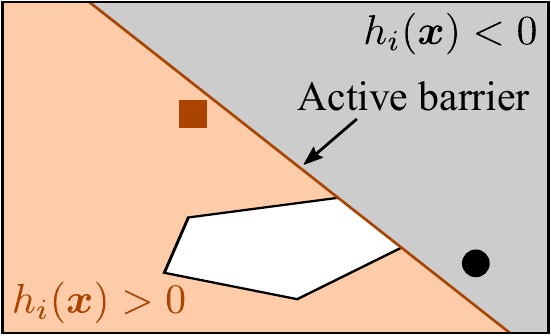}}
    	\caption{(a) The admissible set contains both the starting position and the goal. Trying to avoid a convex region (white area) the goal (square) (b) or the initial position (dot) (c) becomes a non-admissible state. }
    	\label{fig:problem}
\end{figure} 
\hspace*{-1mm}



\section{Conclusions}\label{sec:conclusion}
We presented an approach to learn task space constraints from human demonstrations and to exploit the learned constraints to plan bounded trajectories for robotic manipulators. The approach is incremental and works on-line. Training data are firstly clustered in an unsupervised manner. Points belonging to newly discovered clusters are used to fit planar barrier functions representing the constraints. A previously learned constraint can be easily deleted by pushing the robot beyond the relative barrier. Learned barrier functions are then used to constrain the state of the dynamical system used to plan robot trajectories. We presented an optimization-based approach to compute a smooth control input that forces the DS trajectories to remain within the learned bounds and preserves eventual stability properties. Simulations and experiments show that our approach is effective in learning task space constraints and planning feasible motions. 

{As a future work, we plan to automatically generate training samples using surface-tracking techniques and to consider second-order DS to learn velocity barriers.}



\bibliographystyle{IEEEtran}
\bibliography{bibliography}

\end{document}